\documentclass[11pt]{article}
\usepackage{authblk}

\usepackage[margin=1in]{geometry}
\usepackage[T1]{fontenc}
\usepackage[utf8]{inputenc}
\usepackage{lmodern}
\usepackage{microtype}
\usepackage{setspace}
\usepackage{parskip}

\usepackage{amsmath,amssymb,amsthm,mathtools}
\usepackage{bm}

\newtheorem{theorem}{Theorem}

\newtheorem{proposition}[theorem]{Proposition}
\newtheorem{corollary}[theorem]{Corollary}
\newtheorem{definition}[theorem]{Definition}
\newtheorem{example}[theorem]{Example}
\newtheorem{remark}[theorem]{Remark}

\usepackage{graphicx}
\usepackage{caption}
\usepackage{subcaption}
\usepackage{float}
\usepackage{tikz}
\usetikzlibrary{arrows.meta,positioning,calc}
\usepackage{booktabs}
\usepackage{array}

\usepackage{hyperref}
\hypersetup{
  colorlinks=true,
  linkcolor=blue!70!black,
  citecolor=blue!70!black,
  urlcolor=blue!70!black,
  pdftitle={Operational Non-Commutativity in Metacognitive Evaluation},
  pdfauthor={Enso O. Torres Alegre}
}

\newcommand{\Sspace}{\mathcal{S}}
\newcommand{\Eset}{\mathcal{E}}

\newcommand{\Bspace}{\mathcal{B}}

\newcommand{\id}{\mathrm{id}}
\newcommand{\R}{\mathbb{R}}

\DeclareMathOperator{\E}{\mathbb{E}}

\title{\textbf{Operational Noncommutativity in Sequential Metacognitive Judgments}}

\author[1]{Enso O. Torres Alegre\textsuperscript{*}}
\author[2]{Diana E. Mora Jiménez}

\affil[1]{Pontifical Catholic University of Chile, Santiago, Chile}
\affil[2]{University of Notre Dame, South Bend, IN, United States}

\date{}

\begin{document}
\maketitle

\begin{center}
\small \textsuperscript{*}Corresponding: \texttt{entos930@gmail.com}
\end{center}

\begin{abstract}
Metacognition---the monitoring and regulation of one’s own cognitive processes---is inherently sequential: an agent evaluates an internal state, updates it, and may then re-evaluate under shifted criteria. Order effects in cognition have been widely documented, yet it remains unclear whether such effects merely reflect classical state changes or indicate a deeper structural non-commutativity. We develop an operational framework that makes this distinction explicit. In our formulation, metacognitive evaluations act as state-transforming operations on an internal state space equipped with probabilistic readouts, separating the back-action of evaluation from its observable output.

We show that order dependence obstructs any faithful Boolean--commutative representation, and we address the stronger question of whether observed order effects can always be explained by expanding the state space with classical latent variables. To formalize this issue, we introduce two explicit assumptions---counterfactual definiteness and evaluation non-invasiveness---under which the existence of a joint distribution over all sequential readouts implies a family of testable constraints on pairwise sequential correlations. Violation of these constraints rules out any classical non-invasive account and certifies what we term \emph{genuine non-commutativity}.

We provide an explicit three-dimensional rotation model with fully worked numerical examples that exhibits such violations, and we outline a behavioral paradigm---sequential confidence, error-likelihood, and feeling-of-knowing judgments following a perceptual decision---together with the corresponding empirical test. No claim is made regarding quantum physical substrates; the framework is purely operational and algebraic.
\end{abstract}

\section{Introduction}\label{sec:intro}

The architecture of metacognition is often described through Nelson
and Narens' influential two-level picture \cite{NelsonNarens1990,
NelsonNarens1994}: a \emph{meta-level} that monitors and controls
an \emph{object-level} cognitive process. Computationally, modern
models have cast metacognitive evaluation---confidence ratings,
feelings of knowing, error monitoring---within Bayesian and
signal-detection frameworks with considerable success
\cite{FlemingLau2014, FlemingDaw2017, ManiscalcoLau2012,
Pouget2016}.

A feature that has received comparatively less formal attention is
the \emph{sequentiality} of metacognitive evaluation. In practice
an agent does not simply emit a single self-assessment; rather,
self-evaluation unfolds over time, and importantly, the act of
evaluating can itself transform the state being evaluated. This
observation is not exotic---it is essentially the recognition that
introspection can be \emph{reactive}
\cite{Schooler2002, NisbettWilson1977}. There is now substantial
evidence that metacognitive judgments interact with ongoing
processing: confidence ratings affect subsequent memory encoding
\cite{DoubleBirney2019}, post-decision evidence accumulation
modifies confidence trajectories \cite{Resulaj2009},
and the framing of self-evaluation can alter downstream behavior
\cite{Ais2016}. What has not been clearly laid out, as far as we
know, is the question of what structural consequences follow from
this reactivity when it takes the form of order dependence between
different types of metacognitive evaluation.

The study of order effects in judgment has a large empirical
literature \cite{Trueblood2011, WangBusemeyer2013,
PothosBusemeyer2013}. Busemeyer and Bruza \cite{BusemeyerBruza2012}
and Khrennikov \cite{Khrennikov2010} developed quantum-probability
models of cognition that naturally accommodate non-commutativity.
Our approach shares the mathematical observation that order effects
point toward non-commutative structure, but differs in two respects.
First, we focus specifically on \emph{metacognitive} evaluations---
higher-order assessments of one's own processing---rather than
first-order judgments or preferences. Second, we go beyond the
observation that order effects are incompatible with commutative
frameworks; we develop a criterion, grounded in two explicit
assumptions, for when order effects are \emph{genuinely}
non-commutative rather than merely artifacts of incomplete state
descriptions. This criterion yields testable inequalities whose
violation would constitute evidence for irreducible
non-commutativity in metacognitive evaluation.

The paper is organized as follows. Section~\ref{sec:framework}
lays out the operational framework. Section~\ref{sec:obstruction}
proves the basic obstruction result. Section~\ref{sec:genuine}
develops the distinction between apparent and genuine
non-commutativity, culminating in a family of testable
inequalities. Section~\ref{sec:model} constructs an explicit
minimal model with complete numerical calculations.
Section~\ref{sec:empirical} sketches an experimental paradigm.
Section~\ref{sec:discussion} discusses relations to existing work,
and Section~\ref{sec:conclusion} concludes.

\section{Operational framework}\label{sec:framework}

We aim for a framework that is assumption-light while still having
enough structure to support non-trivial derivations. The key
ingredients are an internal state space, a collection of
metacognitive evaluations acting on it, and readout functions that
yield operationally observable outputs.

\begin{definition}[Internal state space]\label{def:state}
An \emph{internal state space} is a pair $(\Sspace, \Sigma)$
where $\Sspace$ is a nonempty set and $\Sigma$ is a
$\sigma$-algebra on $\Sspace$. An \emph{internal state} is a
probability measure $\mu$ on $(\Sspace,\Sigma)$. We write
$\mathcal{P}(\Sspace)$ for the convex set of all such measures.
\end{definition}

The use of probability measures rather than point-valued states
captures the realistic scenario in which the meta-level does not
have full access to the object-level configuration.

\begin{definition}[Metacognitive evaluation]\label{def:eval}
A \emph{metacognitive evaluation} is a pair $E=(T_E,\pi_E)$
where:
\begin{enumerate}
  \item $T_E:\mathcal{P}(\Sspace)\to\mathcal{P}(\Sspace)$ is a
  measurable state-transformation, representing the back-action
  of the evaluation on the internal state.
  \item $\pi_E:\mathcal{P}(\Sspace)\to[0,1]$ is a measurable
  readout function, representing the operationally accessible
  scalar output (e.g.\ a confidence level or error estimate).
\end{enumerate}
We denote by $\Eset$ a collection of metacognitive evaluations.
\end{definition}

This formulation separates the two aspects of evaluation: producing
an observable output ($\pi_E$) and disturbing the internal state
($T_E$). The disturbance component is precisely what makes
sequential composition nontrivial; without it ($T_E=\id$ for all
$E$) everything would commute trivially.

\begin{definition}[Sequential composition]\label{def:composition}
For $E_1=(T_1,\pi_1)$ and $E_2=(T_2,\pi_2)$ in $\Eset$, the
sequential application ``$E_1$ then $E_2$'' acts on an initial
state $\mu\in\mathcal{P}(\Sspace)$ as follows: first $T_1$ is
applied, yielding $\mu'=T_1(\mu)$ with readout $r_1=\pi_1(\mu)$;
then $T_2$ is applied to $\mu'$, yielding $\mu''=T_2(T_1(\mu))$
with readout $r_2=\pi_2(T_1(\mu))$. We write $E_2\circ E_1$ for
the composite transformation $T_2\circ T_1$, and record the
readout pair $(r_1, r_2)$.
\end{definition}

\begin{definition}[Order dependence]\label{def:orderdep}
$\Eset$ exhibits \emph{order dependence} with respect to a state
$\mu\in\mathcal{P}(\Sspace)$ if there exist $E_1,E_2\in\Eset$ such
that at least one of the following holds:
\begin{enumerate}
  \item[\textbf{(S)}] \textbf{Strong:}
  $T_2(T_1(\mu))\neq T_1(T_2(\mu))$.
  \item[\textbf{(W)}] \textbf{Weak:}
  $T_2(T_1(\mu))= T_1(T_2(\mu))$ but the second-position
  readouts differ: $\pi_2(T_1(\mu))\neq \pi_1(T_2(\mu))$.
\end{enumerate}
\end{definition}

Condition (S) says the final internal state depends on order.
Condition (W) says the states might coincide but the readouts
associated with the second-position evaluation differ depending
on what came first. Both are, in principle, empirically accessible
through sequential-judgment paradigms.

Figure~\ref{fig:concept} summarizes the sequential evaluation structure: an evaluation produces an observable readout while also applying back-action that changes the internal state and can create order effects.

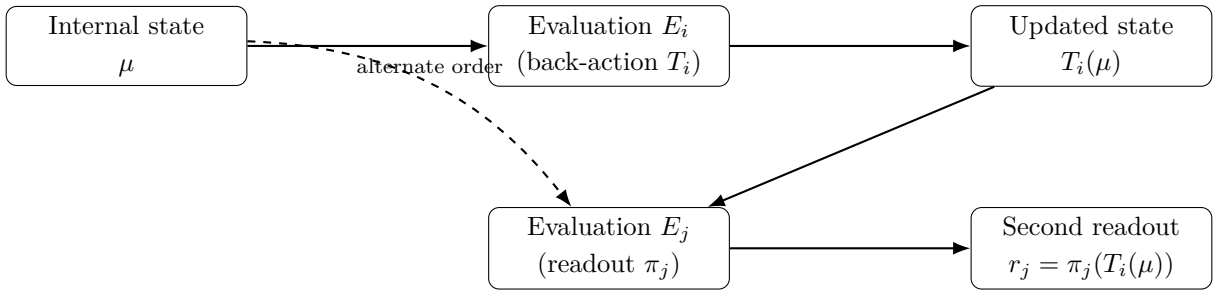
\begin{figure}[H]
\centering
\begin{tikzpicture}[
    node distance=2.0cm,
    every node/.style={font=\small},
    box/.style={rectangle, draw, rounded corners, minimum width=3.2cm, minimum height=0.95cm, align=center},
    arrow/.style={-Latex, thick}
]
\node[box] (mu) {Internal state\\ $\mu$};
\node[box, right=3.2cm of mu] (Ei) {Evaluation $E_i$\\ (back-action $T_i$)};
\node[box, right=3.2cm of Ei] (mu2) {Updated state\\ $T_i(\mu)$};

\node[box, below=1.6cm of Ei] (Ej) {Evaluation $E_j$\\ (readout $\pi_j$)};
\node[box, right=3.2cm of Ej] (rj) {Second readout\\ $r_j=\pi_j(T_i(\mu))$};

\draw[arrow] (mu) -- (Ei);
\draw[arrow] (Ei) -- (mu2);
\draw[arrow] (mu2) -- (Ej);
\draw[arrow] (Ej) -- (rj);

\draw[arrow, dashed] (mu) to[bend left=25] node[above, font=\scriptsize]{alternate order} (Ej);
\end{tikzpicture}
\caption{Sequential metacognitive evaluation as a state-transforming operation with an observable readout. The key empirical object is the dependence of the second-position readout on what came first.}
\label{fig:concept}
\end{figure}

\section{Obstruction to Boolean--commutative representation}
\label{sec:obstruction}

We first establish a baseline result whose content is admittedly
close to definitional, but which serves to fix notation and
clarify exactly which combination of requirements is inconsistent
with order dependence.

\begin{definition}[Boolean--commutative representation]
\label{def:boolrep}
A \emph{Boolean--commutative (BC) representation} of
$(\Sspace,\Sigma,\Eset)$ is a tuple $(\Bspace, \cdot, P, \phi)$
where:
\begin{enumerate}
  \item $\Bspace$ is a Boolean algebra of events.
  \item $\cdot:\Bspace\times\Bspace\to\Bspace$ is a commutative
  binary operation (sequential product): $a\cdot b = b\cdot a$.
  \item $P$ is a probability measure on $\Bspace$.
  \item $\phi:\Eset\to\Bspace$ assigns each evaluation an event
  $\phi(E)=[E]$, with the faithfulness condition: if the
  sequential readout statistics of $E_1$-then-$E_2$ differ
  from $E_2$-then-$E_1$, then $P([E_1]\cdot[E_2])\neq
  P([E_2]\cdot[E_1])$.
\end{enumerate}
\end{definition}

\begin{theorem}[Obstruction]\label{thm:obstruction}
If $(\Sspace,\Sigma,\Eset)$ exhibits order dependence, then no BC
representation exists.
\end{theorem}

\begin{proof}
Commutativity of $\cdot$ forces $[E_1]\cdot[E_2]=[E_2]\cdot[E_1]$,
hence $P([E_1]\cdot[E_2])=P([E_2]\cdot[E_1])$. Order dependence,
combined with the faithfulness requirement, requires these
probabilities to differ. Contradiction.
\end{proof}

\begin{remark}\label{rem:trivial}
As noted, Theorem~\ref{thm:obstruction} is essentially a
reformulation of its hypothesis. Its function in the paper is
to make the logical structure explicit and to motivate the
non-trivial question addressed next: can the order dependence
always be ``explained away'' by expanding the state space?
\end{remark}

\section{Genuine vs.\ apparent non-commutativity}
\label{sec:genuine}

A natural objection to Theorem~\ref{thm:obstruction} is that the
observed order dependence might reflect an incomplete state
description rather than any fundamental non-commutativity: include
the right latent variables, and the evaluations might commute
after all. This objection is serious and must be addressed head-on.
In the quantum foundations literature, the analogous question---
whether quantum correlations can be reproduced by a classical
hidden-variable model---requires substantive no-go theorems
(Bell \cite{Bell1964}, Fine \cite{Fine1982}, Leggett--Garg
\cite{LeggettGarg1985}). We now develop the corresponding
argument for our setting.

\subsection{Two assumptions}\label{subsec:assumptions}

We formulate two explicit assumptions about what a ``classical
account'' of sequential metacognitive evaluation would look like.

\begin{definition}[Counterfactual definiteness (CFD)]
\label{def:cfd}
For each trial (i.e.\ each instance of the agent in state $\mu$)
and each evaluation $E_j\in\Eset$, there exists a definite value
$R_j\in[0,1]$ that the readout \emph{would} take if $E_j$ were
performed, regardless of which other evaluations are or are not
performed on that trial. The agent carries, so to speak, a
complete set of pre-determined readout values
$(R_1, R_2, \ldots, R_n)$ for all evaluations in $\Eset$.
\end{definition}

\begin{definition}[Evaluation non-invasiveness (ENI)]
\label{def:eni}
Performing evaluation $E_i$ does not alter the pre-existing
readout values $R_j$ for $j\neq i$. That is, the ``label'' that
$E_j$ would produce is the same whether or not $E_i$ has been
applied beforehand.
\end{definition}

These two assumptions together constitute what we will call a
\emph{non-invasive classical (NIC) model}. They are the natural
cognitive analogues of the assumptions underlying Leggett--Garg
inequalities in physics \cite{LeggettGarg1985, EmaryLambert2014},
adapted to the metacognitive setting.

\begin{remark}
CFD and ENI are strong assumptions. They are \emph{not} meant as
empirical claims; rather, they define a class of models (NIC
models) which we will then test against data. The point is to
derive consequences of these assumptions that can be empirically
checked. If the consequences are violated, then at least one of
CFD or ENI must fail---which is what we mean by ``genuine
non-commutativity.''
\end{remark}

\subsection{Joint distribution and pairwise inequalities}
\label{subsec:joint}

Under CFD and ENI, a remarkable simplification occurs.

\begin{proposition}[Existence of joint distribution under NIC]
\label{prop:joint}
Assume CFD and ENI hold for evaluations $\{E_1,\ldots,E_n\}
\subseteq\Eset$ and a state $\mu$. Then there exists a joint
probability distribution
$$
  P(R_1, R_2, \ldots, R_n)
$$
over the readout values of all evaluations, such that every
pairwise sequential experiment---performing $E_i$ then
$E_j$---samples from the marginal $P(R_i, R_j)$.
\end{proposition}

\begin{proof}
CFD guarantees that each trial has a well-defined tuple
$(R_1,\ldots,R_n)$. ENI guarantees that the value $R_j$ observed
when $E_j$ is performed second (after $E_i$) is the same $R_j$
that would have been observed had $E_j$ been performed in
isolation or first. Therefore the pair $(R_i,R_j)$ observed in
the $E_i$-then-$E_j$ condition is a sample from a single,
well-defined bivariate marginal of the joint distribution
$P(R_1,\ldots,R_n)$, which exists because each trial defines
the full tuple.
\end{proof}

The existence of the joint distribution has testable consequences.
We now derive these for the case $n=3$ evaluations with
binarized readouts (to keep the combinatorics transparent), and
then state the continuous version.

\subsubsection{Binarized readouts}

Threshold each readout at $1/2$: define $A_j = \mathbf{1}[R_j
\geq 1/2] \in \{0,1\}$. Under the NIC joint distribution, the
pairwise disagreement probabilities
$$
  d_{ij} = P(A_i \neq A_j)
$$
satisfy the triangle inequality for the Hamming pseudo-metric:

\begin{theorem}[NIC inequality, binary case]\label{thm:NIC_binary}
Under CFD and ENI for three evaluations $E_1,E_2,E_3$:
\begin{align}
  d_{12} + d_{23} &\geq d_{13}, \label{eq:tri1}\\
  d_{12} + d_{13} &\geq d_{23}, \label{eq:tri2}\\
  d_{13} + d_{23} &\geq d_{12}. \label{eq:tri3}
\end{align}
\end{theorem}

\begin{proof}
For any three binary random variables $A_1,A_2,A_3$ defined on a
common probability space,
\begin{align*}
  \{A_1\neq A_3\} &\subseteq \{A_1\neq A_2\}\cup\{A_2\neq A_3\}.
\end{align*}
This is because if $A_1\neq A_3$, then either $A_1\neq A_2$ or
$A_2\neq A_3$ (or both); if both $A_1=A_2$ and $A_2=A_3$, then
$A_1=A_3$, a contradiction. Taking probabilities:
$d_{13}\leq d_{12}+d_{23}$. The other two inequalities follow
by relabeling.
\end{proof}

\subsubsection{Continuous readouts}

For continuous readouts we use the mean second-position readout
as the observable quantity. Define
$$
  C_{ij} = \E\bigl[\pi_{E_j}(T_i(\mu))\bigr],
$$
i.e., $C_{ij}$ is the expected readout of $E_j$ when $E_i$ is
performed first. Under a NIC model, ENI implies
$C_{ij}=\E[R_j]$ regardless of $i$; in other words, the mean
readout of $E_j$ does not depend on which evaluation preceded it.

\begin{corollary}[NIC equality, continuous case]
\label{cor:NIC_cont}
Under CFD and ENI, for any $E_i, E_k \in \Eset$ and any $E_j$:
\begin{equation}\label{eq:NIC_eq}
  C_{ij} = C_{kj}.
\end{equation}
\end{corollary}

\begin{proof}
ENI states that performing $E_i$ first does not change the readout
value $R_j$. Hence $\pi_{E_j}(T_i(\mu))=R_j=\pi_{E_j}(T_k(\mu))$
for every trial, and taking expectations gives the result.
\end{proof}

\begin{remark}
Corollary~\ref{cor:NIC_cont} is stronger than the binary
inequalities---it asserts exact equality, not merely an
inequality. In practice, statistical testing would compare
$C_{ij}$ and $C_{kj}$ and reject the NIC hypothesis if they
differ significantly. This is the continuous analogue of the
binary triangle inequalities, and is in fact a more directly
testable prediction.
\end{remark}

\begin{definition}[Genuine non-commutativity, revised]
\label{def:genuine}
Order dependence in $(\Sspace,\Sigma,\Eset)$ is \emph{genuine}
if no NIC model (i.e.\ no model satisfying both CFD and ENI) can
reproduce the observed sequential readout statistics. Empirically,
this is witnessed by either (i) violation of
inequalities~\eqref{eq:tri1}--\eqref{eq:tri3} for binarized
readouts, or (ii) violation of the equalities~\eqref{eq:NIC_eq}
for continuous readouts.
\end{definition}

\begin{remark}\label{rem:whichfails}
When genuine non-commutativity is detected, the framework does not
specify whether CFD or ENI (or both) fails. In the metacognitive
context, failure of ENI---the claim that performing one type of
evaluation changes the readout that a subsequent different
evaluation would produce---is the more natural interpretation, and
essentially formalizes the intuition that ``introspection disturbs
the state being introspected.'' However, failure of CFD is also
possible and would correspond to the idea that metacognitive
readout values do not exist independently of the evaluation context.
\end{remark}

\section{A minimal concrete model}\label{sec:model}

We now construct a model that exhibits genuine non-commutativity,
with all calculations given in full detail.

\begin{example}[Rotation model on $S^2$]\label{ex:rotation}

Let $\Sspace = S^2 \subset \R^3$, the unit sphere. For simplicity
we work with pure (point-valued) states $\bm{v}\in S^2$. Define
three evaluations using rotations in $SO(3)$:
\begin{align*}
  E_1:&\quad T_1 = R_x(\alpha),\qquad
  \pi_1(\bm{v}) = \tfrac{1}{2}(1+v_x),\\
  E_2:&\quad T_2 = R_y(\beta),\qquad
  \pi_2(\bm{v}) = \tfrac{1}{2}(1+v_y),\\
  E_3:&\quad T_3 = R_z(\gamma),\qquad
  \pi_3(\bm{v}) = \tfrac{1}{2}(1+v_z),
\end{align*}
where $R_x(\alpha)$ denotes rotation by angle $\alpha$ about the
$x$-axis, etc. The readout $\pi_j$ extracts the component along
the $j$-th axis, rescaled to $[0,1]$.

\medskip\noindent\textbf{Choice of parameters.}\quad
We take $\alpha=\beta=\gamma=\pi/3$ and $\bm{v}_0 =
\frac{1}{\sqrt{3}}(1,1,1)$. This state is not an eigenvector
of any of the three rotation axes, ensuring that each evaluation
genuinely rotates the state.

\medskip\noindent\textbf{Rotation matrices.}\quad
The standard rotation matrices are
\begin{align}
  R_x(\alpha) &= \begin{pmatrix}
  1 & 0 & 0\\
  0 & \cos\alpha & -\sin\alpha\\
  0 & \sin\alpha & \cos\alpha
  \end{pmatrix},\label{eq:Rx}\\[4pt]
  R_y(\beta) &= \begin{pmatrix}
  \cos\beta & 0 & \sin\beta\\
  0 & 1 & 0\\
  -\sin\beta & 0 & \cos\beta
  \end{pmatrix},\label{eq:Ry}\\[4pt]
  R_z(\gamma) &= \begin{pmatrix}
  \cos\gamma & -\sin\gamma & 0\\
  \sin\gamma & \cos\gamma & 0\\
  0 & 0 & 1
  \end{pmatrix}.\label{eq:Rz}
\end{align}

With $\alpha=\beta=\gamma=\pi/3$, we have $\cos(\pi/3)=1/2$ and
$\sin(\pi/3)=\sqrt{3}/2$. Set $c = 1/\sqrt{3}\approx 0.5774$.

\medskip\noindent\textbf{Computing $T_1(\bm{v}_0)=
R_x(\pi/3)\,\bm{v}_0$.}\quad
\begin{align*}
  R_x(\pi/3)\begin{pmatrix}c\\c\\c\end{pmatrix}
  &= \begin{pmatrix}
  c\\
  c\cdot\frac{1}{2} - c\cdot\frac{\sqrt{3}}{2}\\[2pt]
  c\cdot\frac{\sqrt{3}}{2} + c\cdot\frac{1}{2}
  \end{pmatrix}
  = \begin{pmatrix}
  c\\[2pt]
  c\,\frac{1-\sqrt{3}}{2}\\[2pt]
  c\,\frac{\sqrt{3}+1}{2}
  \end{pmatrix}.
\end{align*}
Numerically:
$T_1(\bm{v}_0) \approx (0.5774,\; -0.2113,\; 0.7887)$.

\medskip\noindent\textbf{Computing $T_2(\bm{v}_0)=
R_y(\pi/3)\,\bm{v}_0$.}\quad
\begin{align*}
  R_y(\pi/3)\begin{pmatrix}c\\c\\c\end{pmatrix}
  &= \begin{pmatrix}
  c\cdot\frac{1}{2} + c\cdot\frac{\sqrt{3}}{2}\\[2pt]
  c\\[2pt]
  -c\cdot\frac{\sqrt{3}}{2} + c\cdot\frac{1}{2}
  \end{pmatrix}
  = \begin{pmatrix}
  c\,\frac{1+\sqrt{3}}{2}\\[2pt]
  c\\[2pt]
  c\,\frac{1-\sqrt{3}}{2}
  \end{pmatrix}.
\end{align*}
Numerically:
$T_2(\bm{v}_0) \approx (0.7887,\; 0.5774,\; -0.2113)$.

\medskip\noindent\textbf{Computing $T_3(\bm{v}_0)=
R_z(\pi/3)\,\bm{v}_0$.}\quad
\begin{align*}
  R_z(\pi/3)\begin{pmatrix}c\\c\\c\end{pmatrix}
  &= \begin{pmatrix}
  c\cdot\frac{1}{2} - c\cdot\frac{\sqrt{3}}{2}\\[2pt]
  c\cdot\frac{\sqrt{3}}{2} + c\cdot\frac{1}{2}\\[2pt]
  c
  \end{pmatrix}
  = \begin{pmatrix}
  c\,\frac{1-\sqrt{3}}{2}\\[2pt]
  c\,\frac{\sqrt{3}+1}{2}\\[2pt]
  c
  \end{pmatrix}.
\end{align*}
Numerically:
$T_3(\bm{v}_0) \approx (-0.2113,\; 0.7887,\; 0.5774)$.

\medskip\noindent\textbf{Sequential readouts.}\quad
We now compute the six pairwise sequential readouts $C_{ij} =
\pi_j(T_i(\bm{v}_0))$:

\begin{center}
\renewcommand{\arraystretch}{1.3}
\begin{tabular}{@{}ccc@{}}
\toprule
Sequence & $C_{ij}$ (exact) & $C_{ij}$ (numerical)\\
\midrule
$E_1$ then $E_2$:\; $C_{12}=\pi_2(T_1(\bm{v}_0))$
  & $\frac{1}{2}\bigl(1+c\,\frac{1-\sqrt{3}}{2}\bigr)$
  & $\approx 0.3944$\\[3pt]
$E_1$ then $E_3$:\; $C_{13}=\pi_3(T_1(\bm{v}_0))$
  & $\frac{1}{2}\bigl(1+c\,\frac{\sqrt{3}+1}{2}\bigr)$
  & $\approx 0.8944$\\[3pt]
$E_2$ then $E_1$:\; $C_{21}=\pi_1(T_2(\bm{v}_0))$
  & $\frac{1}{2}\bigl(1+c\,\frac{1+\sqrt{3}}{2}\bigr)$
  & $\approx 0.8944$\\[3pt]
$E_2$ then $E_3$:\; $C_{23}=\pi_3(T_2(\bm{v}_0))$
  & $\frac{1}{2}\bigl(1+c\,\frac{1-\sqrt{3}}{2}\bigr)$
  & $\approx 0.3944$\\[3pt]
$E_3$ then $E_1$:\; $C_{31}=\pi_1(T_3(\bm{v}_0))$
  & $\frac{1}{2}\bigl(1+c\,\frac{1-\sqrt{3}}{2}\bigr)$
  & $\approx 0.3944$\\[3pt]
$E_3$ then $E_2$:\; $C_{32}=\pi_2(T_3(\bm{v}_0))$
  & $\frac{1}{2}\bigl(1+c\,\frac{\sqrt{3}+1}{2}\bigr)$
  & $\approx 0.8944$\\
\bottomrule
\end{tabular}
\end{center}

\medskip\noindent\textbf{Verifying order dependence.}\quad
Compare $C_{12}\approx 0.3944$ with $C_{21}\approx 0.8944$.
Under the NIC equality (Corollary~\ref{cor:NIC_cont}), $C_{12}$
should equal $C_{32}$ (both are second-position readouts of
$E_2$). We have $C_{12}\approx 0.3944$ and $C_{32}\approx 0.8944$:
a clear violation. Similarly, $C_{13}\approx 0.8944 \neq C_{23}
\approx 0.3944$, and $C_{21}\approx 0.8944 \neq C_{31}\approx
0.3944$.

\medskip\noindent\textbf{Checking the binary triangle
inequality.}\quad For completeness, we also verify
violation of \eqref{eq:tri1}--\eqref{eq:tri3} under binarization
at $1/2$. For this deterministic (pure-state) model,
$A_j=\mathbf{1}[\pi_j(\bm{v})\geq 1/2]$, so:
\begin{itemize}
  \item At $\bm{v}_0$: $\pi_1=\pi_2=\pi_3 = \frac{1}{2}(1+c)
  \approx 0.789$, so $A_1=A_2=A_3=1$ and all $d_{ij}=0$.
  No violation (trivially satisfied).
\end{itemize}
This illustrates that the continuous NIC equalities are
strictly more informative than the binary inequalities: the binary
version has no power for a deterministic single-state model, while
the continuous equalities are sharply violated.
\end{example}

\begin{remark}\label{rem:model_role}
The rotation model is \emph{not} proposed as a neural mechanism.
Its purpose is to serve as an existence proof: there exist simple,
finite-dimensional, fully computable systems with genuine
non-commutativity in the sense of Definition~\ref{def:genuine}.
One could interpret the three axes as ``dimensions'' of
metacognitive assessment (confidence, error-likelihood,
feeling-of-knowing) and the rotations as the back-action of
performing each type of evaluation---a back-action that ``mixes''
the dimensions. But this interpretation is a heuristic, not a
mechanistic claim.
\end{remark}

We note that the event structure generated by the projection-valued
readouts on $S^2$ under non-commuting rotations is naturally
described by an orthomodular lattice rather than a Boolean algebra
\cite{BirkhoffvonNeumann1936, Svozil1998}. This suggests---but
does not prove in general---that the appropriate replacement for
the Boolean event algebra in the presence of genuine
non-commutativity may be a lattice that is orthomodular but not
distributive.

\section{Empirical paradigm and testable predictions}
\label{sec:empirical}

\subsection{Paradigm sketch}

Consider a perceptual discrimination task (e.g.\ random-dot
motion coherence discrimination \cite{FlemingDaw2017}). On each
trial, after the primary decision, the participant performs
\emph{two} sequential metacognitive evaluations drawn from:
\begin{itemize}
  \item $E_C$: ``Rate your confidence that you were correct''
  (continuous scale, 0--100, rescaled to $[0,1]$).
  \item $E_L$: ``Rate the likelihood that you made an error''
  (0--100).
  \item $E_K$: ``Rate your feeling-of-knowing about the stimulus
  identity'' (0--100).
\end{itemize}
The order of the two evaluations is randomized across trials,
yielding six conditions ($E_C$-then-$E_L$, $E_L$-then-$E_C$,
etc.). Each condition provides an estimate of the corresponding
$C_{ij}$.

\subsection{Primary prediction: NIC equality test}

The most direct test is Corollary~\ref{cor:NIC_cont}. Under a
NIC model, the mean readout of $E_j$ in second position should
not depend on which evaluation came first. Concretely:
\begin{equation}\label{eq:emp_test}
  C_{CL} = C_{KL},\qquad
  C_{CK} = C_{LK},\qquad
  C_{LC} = C_{KC}.
\end{equation}
A statistically significant difference in any of these pairs---
after controlling for fatigue, anchoring, and practice effects
through counterbalancing and hierarchical regression on trial
order, response time, and accuracy---would reject the NIC
hypothesis.

\subsection{Secondary prediction: binary triangle inequality}

If readouts are binarized (e.g.\ ``high confidence'' vs.\ ``low
confidence'' at the median split), the disagreement probabilities
$d_{ij}$ can be estimated across trials, and the triangle
inequalities~\eqref{eq:tri1}--\eqref{eq:tri3} can be tested
directly. As noted in Example~\ref{ex:rotation}, this test has
less power than the continuous equality test but provides a
complementary check with minimal distributional assumptions.

\subsection{Controlling for classical confounds}

The critical question is whether any observed violation of
\eqref{eq:emp_test} could be absorbed by classical covariates.
The NIC framework provides a clear answer: under CFD+ENI,
the equalities must hold \emph{regardless} of what classical
covariates are included, because the joint distribution is
guaranteed to exist by Proposition~\ref{prop:joint}. Including
covariates (response time, accuracy, stimulus difficulty, trial
number, etc.) as additional conditioning variables can only
\emph{sharpen} the test by reducing noise; it cannot make the
NIC equalities hold if the underlying process is genuinely
non-commutative.

Of course, one could always abandon ENI and claim that evaluation
$E_i$ changes the \emph{classical} state in a way that affects
$R_j$; this is precisely what it means to reject the NIC model.
The point is that such state-disruption, if it is
order-dependent, constitutes the operational non-commutativity
that our framework describes.

\section{Discussion}\label{sec:discussion}

\subsection{Relation to Bayesian metacognition}

Bayesian models of confidence \cite{FlemingDaw2017, Pouget2016,
Sanders2016} treat metacognitive evaluation as approximate
inference over an internal generative model. In these models the
evaluation is typically treated as a passive readout: computing
a confidence estimate does not change the underlying decision
variable. Our framework is fully compatible with the Bayesian
approach in this regime. When evaluations have trivial
back-action ($T_E=\id$), ENI holds automatically and all NIC
equalities are trivially satisfied.

The interesting regime is when metacognitive evaluation is not
passive but involves active processing---re-sampling from memory,
reallocating attention, constructing explicit justifications---
that alters the cognitive state. There is growing evidence that
this regime is common: confidence judgments affect subsequent
memory \cite{DoubleBirney2019}, post-decision evidence
accumulation modifies confidence trajectories in ways that
depend on the timing and type of the metacognitive query
\cite{Resulaj2009, vandenBerg2016}, and changes
of mind interact non-trivially with metacognitive reports
\cite{Fleming2012}. Our framework formalizes the structural
consequences of such active evaluation.

The conceptual contribution of this framework lies not in proposing a specific cognitive mechanism, but in providing a structural criterion that distinguishes classical invasive explanations of order effects from irreducibly non-commutative sequential dynamics. Future work could integrate this operational criterion with Bayesian and dynamical models of confidence formation, examine whether empirical violations arise in controlled paradigms, and explore whether similar constraints apply to other forms of higher-order cognition beyond metacognitive judgment.

\subsection{Relation to quantum cognition}

The quantum cognition program \cite{BusemeyerBruza2012,
PothosBusemeyer2013, Khrennikov2010} uses quantum probability
to model various cognitive phenomena including order effects in
judgments and decisions \cite{WangBusemeyer2013, Trueblood2011}.
Our work shares the mathematical toolkit of non-commutative
algebras and the general strategy of deriving testable
inequalities. Key differences include:

\begin{enumerate}
  \item We focus on metacognitive evaluations rather than
  first-order judgments. The back-action of self-evaluation
  on the evaluated state is, we argue, a more natural setting
  for non-commutativity than preference reversals or
  conjunction fallacies.

  \item We do not presuppose a Hilbert-space formalism. Our
  definitions are set-theoretic and our NIC assumptions
  (CFD+ENI) are stated independently of any quantum apparatus.
  The rotation model of Section~\ref{sec:model} happens to
  embed in a Hilbert space, but this is a property of the
  example, not of the general theory.

  \item We explicitly separate the question ``are there order
  effects?'' from ``are they genuinely non-commutative?'' via
  the NIC criterion.
\end{enumerate}

\subsection{Relation to introspection-as-disturbance}

The philosophical idea that introspection modifies the state
being introspected has a long history
\cite{NisbettWilson1977, Schwitzgebel2008} and has been
discussed in cognitive psychology under headings like ``verbal
overshadowing'' and reactive self-monitoring
\cite{Schooler2002, Proust2013}. Our framework can be read as
providing a formal backbone for this intuition: the back-action
maps $T_E$ are the mathematical expression of
introspection-induced disturbance. What we add is a
criterion---the NIC equalities---for when such disturbance is
non-trivial in the specific sense that it cannot be accounted for
by any classical state-expansion satisfying CFD and ENI.

\subsection{Limitations}

Several limitations deserve explicit mention.

\begin{enumerate}
  \item \textbf{NIC vs.\ weaker classical models.}\quad
  Our inequalities certify the failure of CFD+ENI jointly.
  One might consider weaker forms of classicality that
  abandon ENI but retain some form of hidden-variable
  structure. Whether useful constraints can be derived
  against such weaker models is an open question.

  \item \textbf{Model-specificity.}\quad
  The rotation model of Section~\ref{sec:model} is a proof
  of concept. Fitting it to real metacognitive data would
  require specifying which ``dimensions'' correspond to which
  cognitive variables---a non-trivial modeling step that
  we have not attempted here.

  \item \textbf{Noise and power.}\quad
  In practice, metacognitive reports are noisy, and violations
  of the NIC equalities may be small. Adequate power will
  require either large samples or paradigms that maximize
  the back-action of evaluation. Adaptive experimental designs
  that estimate the $C_{ij}$ online and allocate trials to
  the most informative conditions could help.

  \item \textbf{Stationarity.}\quad
  We assume the initial state $\mu$ is the same across
  trials. In practice, internal states fluctuate. Violation
  of stationarity could mimic or mask genuine
  non-commutativity. Within-trial designs (where both
  evaluations happen on the same trial) mitigate this, but
  introduce concerns about fatigue and response dependencies.
\end{enumerate}

\section{Conclusion}\label{sec:conclusion}

We have developed an operational framework for studying the
sequential structure of metacognitive evaluation. The
contributions are at three levels:

\begin{enumerate}
  \item \textbf{Formal separation.}\quad
  Metacognitive evaluations are decomposed into a back-action
  component ($T_E$) and a readout component ($\pi_E$). This
  separation cleanly captures the intuition that self-evaluation
  can disturb the state being evaluated.

  \item \textbf{Testable criterion for genuine
  non-commutativity.}\quad
  Under two explicit assumptions (CFD and ENI), the existence
  of a joint distribution over all readout values implies
  equalities~\eqref{eq:NIC_eq} that are empirically testable.
  Violation rejects the joint existence and certifies genuine
  non-commutativity.

  \item \textbf{Existence proof and experimental sketch.}\quad
  A minimal rotation model demonstrates that genuine
  non-commutativity is achievable in a simple
  finite-dimensional system, and a behavioral paradigm
  involving sequential metacognitive ratings is outlined.
\end{enumerate}

We emphasize what the framework does \emph{not} claim: it does
not assert that neural systems implement quantum physical
dynamics, nor does it privilege any particular algebraic
replacement for the Boolean event structure. The non-commutative
structure identified here is a descriptive feature of certain
sequential evaluation processes. Whether it is empirically
realized in human metacognition is an open question that the
proposed experimental paradigm is designed to answer.

\appendix
\section{Detailed derivation of NIC constraints}
\label{app:derivation}

We provide a self-contained derivation of the NIC equalities and
binary inequalities.

\subsection{Setup}

Consider a set of evaluations $\{E_1,\ldots,E_n\}$ and a fixed
initial state $\mu$. On each trial, the agent is prepared in
state $\mu$ and two evaluations $E_i, E_j$ are performed
sequentially. The experimenter observes the readout pair
$(r_i, r_j)$ where $r_i = \pi_i(\mu)$ (first-position readout)
and $r_j = \pi_j(T_i(\mu))$ (second-position readout).

\subsection{NIC equalities (continuous)}

Under CFD, there exist definite values $R_1(\omega),\ldots,
R_n(\omega)$ for each trial $\omega$ in a sample space $\Omega$.
Under ENI, the observed second-position readout satisfies
\[
  r_j = R_j(\omega) \quad\text{for all } i,
\]
i.e.\ it equals the pre-determined value regardless of which
evaluation preceded it. Therefore
\[
  \E[r_j \mid E_i\text{ first}] = \E[R_j] =
  \E[r_j \mid E_k\text{ first}]
\]
for all $i,k$, which is~\eqref{eq:NIC_eq}.

\subsection{Binary inequalities}

Binarize: $A_j(\omega) = \mathbf{1}[R_j(\omega)\geq 1/2]$.
Since all $A_j$ are defined on the same probability space
(by CFD), for any three evaluations:
\begin{align*}
  P(A_1\neq A_3)
  &= P(A_1\neq A_3,\, A_1=A_2) + P(A_1\neq A_3,\, A_1\neq A_2)\\
  &\leq P(A_2\neq A_3) + P(A_1\neq A_2),
\end{align*}
where the inequality in the first term uses: if $A_1=A_2$ and
$A_1\neq A_3$, then $A_2\neq A_3$. This gives $d_{13}\leq
d_{12}+d_{23}$; the other forms follow by permutation. \qed

\subsection{Relation to Fine's theorem}

Fine's theorem \cite{Fine1982} states that, for dichotomic
observables, the existence of a joint distribution reproducing
all pairwise marginals is equivalent to a set of Bell-type
inequalities. Our binary inequalities
\eqref{eq:tri1}--\eqref{eq:tri3} are the $n=3$ case of Fine's
conditions, specialized to the Hamming distance formulation. The
continuous equalities~\eqref{eq:NIC_eq} do not follow from Fine
directly but from the stronger ENI assumption, which implies
not just marginal compatibility but actual value-determinism of
the second-position readout.

\section*{Acknowledgments}

Language editing and stylistic refinement were assisted by a large language model (ChatGPT 5.2). All conceptual development, formal assumptions, definitions, results, proofs, modeling, and reference verification were performed and reviewed by the author, who assumes full responsibility for the content.



\begin{thebibliography}{99}

\bibitem{NelsonNarens1990}
T.~O. Nelson and L.~Narens,
``Metamemory: A theoretical framework and new findings,''
in \emph{The Psychology of Learning and Motivation}, Vol.~26,
ed.\ G.~H.~Bower (Academic Press, 1990), pp.~125--173.

\bibitem{NelsonNarens1994}
T.~O. Nelson and L.~Narens,
``Why investigate metacognition?''
in \emph{Metacognition: Knowing about Knowing},
eds.\ J.~Metcalfe and A.~P.~Shimamura (MIT Press, 1994),
pp.~1--25.

\bibitem{FlemingLau2014}
S.~M. Fleming and H.~C. Lau,
``How to measure metacognition,''
\emph{Frontiers in Human Neuroscience} \textbf{8}, 443 (2014).

\bibitem{FlemingDaw2017}
S.~M. Fleming and N.~D. Daw,
``Self-evaluation of decision-making: A general Bayesian framework
for metacognitive computation,''
\emph{Psychological Review} \textbf{124}, 91--114 (2017).

\bibitem{ManiscalcoLau2012}
B.~Maniscalco and H.~Lau,
``A signal detection theoretic approach for estimating metacognitive
sensitivity from confidence ratings,''
\emph{Consciousness and Cognition} \textbf{21}, 422--430 (2012).

\bibitem{Pouget2016}
A.~Pouget, J.~Drugowitsch, and A.~Kepecs,
``Confidence and certainty: distinct probabilistic quantities for
different goals,''
\emph{Nature Neuroscience} \textbf{19}, 366--374 (2016).

\bibitem{Sanders2016}
J.~I. Sanders, B.~Hangya, and A.~Kepecs,
``Signatures of a statistical computation in the human sense of
confidence,''
\emph{Neuron} \textbf{90}, 499--506 (2016).

\bibitem{BusemeyerBruza2012}
J.~R. Busemeyer and P.~D. Bruza,
\emph{Quantum Models of Cognition and Decision}
(Cambridge University Press, 2012).

\bibitem{PothosBusemeyer2013}
E.~M. Pothos and J.~R. Busemeyer,
``Can quantum probability provide a new direction for cognitive
modeling?''
\emph{Behavioral and Brain Sciences} \textbf{36}, 255--274
(2013).

\bibitem{WangBusemeyer2013}
Z.~Wang and J.~R. Busemeyer,
``A quantum question order model supported by empirical tests of
an \emph{a priori} and precise predictions,''
\emph{Topics in Cognitive Science} \textbf{5}, 689--710 (2013).

\bibitem{Khrennikov2010}
A.~Khrennikov,
\emph{Ubiquitous Quantum Structure: From Psychology to Finance}
(Springer, 2010).

\bibitem{Trueblood2011}
J.~S. Trueblood and J.~R. Busemeyer,
``A quantum probability account of order effects in inference,''
\emph{Cognitive Science} \textbf{35}, 1518--1552 (2011).

\bibitem{Fine1982}
A.~Fine,
``Hidden variables, joint probability, and the Bell inequalities,''
\emph{Physical Review Letters} \textbf{48}, 291--295 (1982).

\bibitem{Bell1964}
J.~S. Bell,
``On the Einstein Podolsky Rosen paradox,''
\emph{Physics Physique Fizika} \textbf{1}, 195--200 (1964).

\bibitem{LeggettGarg1985}
A.~J. Leggett and A.~Garg,
``Quantum mechanics versus macroscopic realism: Is the flux there
when nobody looks?''
\emph{Physical Review Letters} \textbf{54}, 857--860 (1985).

\bibitem{EmaryLambert2014}
C.~Emary, N.~Lambert, and F.~Nori,
``Leggett--Garg inequalities,''
\emph{Reports on Progress in Physics} \textbf{77}, 016001 (2014).

\bibitem{BirkhoffvonNeumann1936}
G.~Birkhoff and J.~von Neumann,
``The logic of quantum mechanics,''
\emph{Annals of Mathematics} \textbf{37}, 823--843 (1936).

\bibitem{Svozil1998}
K.~Svozil,
\emph{Quantum Logic} (Springer, 1998).

\bibitem{NisbettWilson1977}
R.~E. Nisbett and T.~D. Wilson,
``Telling more than we can know: Verbal reports on mental
processes,''
\emph{Psychological Review} \textbf{84}, 231--259 (1977).

\bibitem{Schooler2002}
J.~W. Schooler,
``Re-representing consciousness: Dissociations between experience
and meta-consciousness,''
\emph{Trends in Cognitive Sciences} \textbf{6}, 339--344 (2002).

\bibitem{Schwitzgebel2008}
E.~Schwitzgebel,
``The unreliability of naive introspection,''
\emph{Philosophical Review} \textbf{117}, 245--273 (2008).

\bibitem{Proust2013}
J.~Proust,
\emph{The Philosophy of Metacognition: Mental Agency and
Self-Awareness} (Oxford University Press, 2013).

\bibitem{DoubleBirney2019}
K.~S. Double and D.~P. Birney,
``Reactivity to confidence ratings in older and younger adults,''
\emph{Metacognition and Learning} \textbf{14}, 165--182 (2019).

\bibitem{Ais2016}
J.~Ais, A.~Zylberberg, P.~Barttfeld, and M.~Sigman,
``Individual consistency in the accuracy and distribution of
confidence judgments,''
\emph{Cognition} \textbf{146}, 377--386 (2016).

\bibitem{Resulaj2009}
A.~Resulaj, R.~Kiani, D.~M. Wolpert, and M.~N. Shadlen,
``Changes of mind in decision-making,''
\emph{Nature} \textbf{461}, 263--266 (2009).

\bibitem{vandenBerg2016}
R.~van den Berg, K.~Anandalingam, A.~Zylberberg, R.~Kiani,
M.~N. Shadlen, and D.~M. Wolpert,
``A common mechanism underlies changes of mind about decisions
and confidence,''
\emph{eLife} \textbf{5}, e12192 (2016).


\bibitem{Fleming2012}
S.~M. Fleming, R.~J. Dolan, and C.~D. Frith,
``Metacognition: computation, biology and function,''
\emph{Philosophical Transactions of the Royal Society B}
\textbf{367}, 1280--1286 (2012).

\end{thebibliography}
\end{document}